\documentclass{article}
\usepackage{ijcai19}

\usepackage{microtype}
\usepackage{graphicx}
\usepackage{subfigure}
\usepackage{booktabs} 
\usepackage{bm} 
\usepackage{amsmath}
\usepackage{amsthm}
\usepackage{amssymb}
\usepackage{bbm}
\usepackage{xcolor}
\usepackage{algorithm}
\usepackage{algorithmic}
\usepackage{epsdice}
\usepackage[algo2e, noend, noline, ruled, linesnumbered]{algorithm2e}

\providecommand{\DontPrintSemicolon}{\dontprintsemicolon}
\DontPrintSemicolon

\usepackage{dsfont}

\newcommand{\beq}{\begin{equation}}
\newcommand{\eeq}{\end{equation}}

\newcommand{\beqa}{\begin{eqnarray}}
\newcommand{\eeqa}{\end{eqnarray}}

\newcommand{\beqan}{\begin{eqnarray*}}
\newcommand{\eeqan}{\end{eqnarray*}}

\newcommand{\argmin}{\operatornamewithlimits{argmin}}
\newcommand{\bE}{\mathbb{E}}

\newcommand{\bpi}{{\bm\pi}}

\newcommand{\bq}{\mathbf{q}}

\newcommand{\btheta}{\boldsymbol\theta}
\newcommand{\cA}{\mathcal{A}}
\newcommand{\cB}{\mathcal{B}}

\newcommand{\cS}{\mathcal{S}}

\newcommand{\cZ}{\mathcal{Z}}

\newcommand{\defword}[1]{\textbf{\boldmath{#1}}}
\newcommand{\ie}{{\it i.e.}~}

\newtheorem{definition}{Definition}

\newtheorem{theorem}{Theorem}
\newtheorem{corollary}{Corollary}
\newtheorem{lemma}{Lemma}

\newcommand{\abs}[1]{\left|#1\right|}

\newtheorem{remark}{Remark}
\newcommand{\subdiff}{\partial}
\newcommand{\grad}{\nabla}
\newcommand{\as}{\doteq}
\newcommand{\bec}{\bm}
\newcommand{\BrSet}[1]{\textsc{BR}(#1)}
\newcommand{\NashConv}{\textsc{NashConv}}

\newcommand{\subex}[1]{\left(#1\right)}
\newcommand{\subblock}[1]{\left[#1\right]}

\newcommand{\ip}[2]{\langle #1, #2 \rangle}
\newcommand{\reals}{\mathds{R}}

\newcommand{\indicator}[1]{\mathds{1}_{#1}}
\newcommand{\e}[1]{e^{#1}}

\newcommand{\utilityRange}{\Delta}
\newcommand{\policy}{\pi}
\newcommand{\expertValue}{\upsilon}
\newcommand{\expertValueSet}{\Upsilon}
\newcommand{\eExpertValue}{\bar{\expertValue}}
\newcommand{\Actions}{\cA}
\newcommand{\action}{a}
\newcommand{\numActions}{\abs{\Actions}}
\newcommand{\temperature}{\tau}
\newcommand{\gaVariable}{\theta}
\newcommand{\prediction}{x}
\newcommand{\stepSize}{\alpha}
\newcommand{\projection}{\Pi}
\newcommand{\softmaxTransfer}{\projection_{\text{sm}}}

\usepackage{times}
\usepackage{soul}
\usepackage{url}
\usepackage[utf8]{inputenc}
\usepackage[small]{caption}
\usepackage{graphicx}
\usepackage{amsmath}
\usepackage{booktabs}
\definecolor{darkgreen}{RGB}{0,125,0}
\definecolor{darkblue}{RGB}{0,0,125}

\newcounter{mlNoteCounter}
\newcounter{jpNoteCounter}
\newcounter{jblNoteCounter}
\newcounter{ktNoteCounter}
\newcounter{elNoteCounter}
\newcounter{dmNoteCounter}

\title{Computing Approximate Equilibria in Sequential\\Adversarial Games by Exploitability Descent}

\author{
Edward Lockhart$^1$ \and Marc Lanctot$^1$ \and Julien P\'{e}rolat$^1$ \and\\ Jean-Baptiste Lespiau$^1$ \and Dustin Morrill$^{1,2}$ \and Finbarr Timbers$^1$ \and Karl Tuyls$^1$\\
\affiliations
$^1$DeepMind\\
$^2$University of Alberta, Edmonton, Canada\\
\emails
\{locked, lanctot, perolat, jblespiau, dmorrill, finbarrtimbers, karltuyls\}@google.com, morrill@ualberta.ca\\
}

\begin{document}

\maketitle

\begin{abstract}
In this paper, we present {\it exploitability descent}, a new algorithm to compute
approximate equilibria in two-player zero-sum extensive-form games with imperfect information,
by direct policy optimization against worst-case opponents. We prove that when following this optimization, when using counterfactual values, the exploitability of a player's strategy converges asymptotically to zero, and hence when both players employ this optimization, the joint policies converge to a Nash equilibrium.
Unlike fictitious play (XFP) and counterfactual regret minimization (CFR), our convergence result pertains to the policies being optimized rather than the average policies.
Our experiments demonstrate convergence rates comparable to XFP and CFR in four benchmark games in the tabular case. Using function approximation, we find that our algorithm outperforms the tabular version in two of the games, which, to the best of our knowledge, is the first such result in imperfect information games among this class of algorithms.
\end{abstract}


\section{Introduction}

Extensive-form games model sequential interactions between multiple agents, each of which 
maximize their own utility. Classic examples are perfect information games (e.g. chess and Go),
which have served as milestones for measuring the progress of artificial
intelligence~\cite{Campbell02deepblue,Silver16Go}.
When there are simultaneous moves, such as in Markov games, the players may need stochastic policies to guarantee their worst-case expected utility, and must use linear programming at each state for value back-ups. Computing policies for imperfect information games is much more difficult: no Bellman operator exists, so approximate dynamic programming is not applicable; exact equilibrium solutions can be found by sequence-form linear programming~\cite{SequenceFormLPs,Shoham09}, but
these techniques do not scale to very large games.

The challenge domain for imperfect information has been computer Poker, which has driven much of the
progress in computational approaches to equilibrium-finding~\cite{Rubin11Poker}. While there are 
gradient descent techniques that can find an $\epsilon$-Nash equilibrium in $O(\frac{1}{\epsilon})$
iterations~\cite{EGT-OO}, the dominant technique has been counterfactual regret minimization
(CFR)~\cite{CFR}. 
Based on CFR, recent techniques have solved heads-up limit Texas Hold'em~\cite{Bowling15Poker} and beat human professionals in no-limit Texas Hold'em~\cite{Moravcik17DeepStack,Brown17Libratus}.

Other techniques have emerged in recent years, based first on fictitious play (XFP)~\cite{Heinrich15FSP},
and generalized to double oracle and any meta-game solver over sets of policies~\cite{Lanctot17PSRO}. Both require a subroutine that computes a best response (an ``oracle''). Here, reinforcement learning can be used to compute approximate oracles, and function approximation can be used to generalize over the state space without domain-specific abstraction mechanisms. Hence, deep neural networks can trained from zero knowledge as in AlphaZero~\cite{Silver18AlphaZero}. 
Policy gradient techniques are also compatible with function approximation in this setting~\cite{srinivasan2018actor}. Combining data buffers with CFR using regression to predict regrets has also shown promise in medium-sized poker variants~\cite{Waugh15solving,Brown18DeepCFR}.

In this paper, we introduce a new algorithm for computing approximate Nash equilibria. Like XFP, best responses are computed at each iteration. Unlike XFP, players optimize
their policies directly against their worst-case opponent. 
When using tabular policies and $\ell_2$ projections after policy updates, the sequence of policies will contain an $\epsilon$-Nash equilibrium, unlike CFR and XFP that only converge-in-average.
Our algorithm works well with function approximation, as the problem can be expressed directly as a policy gradient optimization. Our experiments show convergence rates comparable to XFP and CFR in the tabular setting, exhibit generalization over the state space using neural networks in four different games.

At the time of original submission, we were unaware of a similar algorithm recently presented at the Deep RL Workshop NeurIPS 2018: Self-Play Against a Best Response (SPAR)~\cite{Tang18SPAR}. The work we present in this paper was done independently.
In this paper, we provide convergence guarantees, as well as results in both the tabular and neural network cases. We do so on four benchmark games (both commonly used poker benchmarks used in~\cite{Tang18SPAR} and two additional games), whereas results of SPAR focus on the sample-based setting not covered in this paper.

\section{Background and Terminology}
\label{sec:background}

An \defword{extensive-form game} describes a sequential interaction between \defword{players} $i \in \{ 1, 2, \cdots, n\} \cup \{ c \}$, where $c$ is considered a special player called \defword{chance} with a fixed stochastic policy that determines the transition probabilities given states and actions.
We will often use $-i$ to refer to all the opponents of $i$. In this paper, we focus on the $n=2$ player setting.

The game starts in the empty history $h = \emptyset$. On each turn, a player $i$ chooses an action $a \in \cA_i$,
changing the history to $h' = ha$.
Here $h$ is called a prefix history of $h'$, denoted $h \sqsubset h'$.
The full history is sometimes also called a {\it ground} state because it uniquely identifies the true state, since chance's actions are included.
In poker, for example, a ground state would include all the players' private cards.
We define an \defword{information state} $s \in \cS$ for player $i$ as the state as perceived by an agent which is consistent with its observations.
Formally, each $s$ is a set of histories, specifically $h, h' \in s \Leftrightarrow$ the sequence of of player $i$'s observations along $h$ and $h'$ are equal.
In poker, an information state groups together all the histories that differ only in the private cards of $-i$. Denote $\cZ$ the set of terminal histories, each corresponding to the end of a game, and a utility to each player $u_i(z)$ for $z \in \cZ$.
We also define $\tau(s)$ as the player whose turn it is at $s$,
and $\cZ(h)$ the subset of terminal histories that share $h$ as a prefix.

Since players cannot observe the ground state $h$, policies are defined as $\pi_i : \cS_i \rightarrow \Delta(\cA_i)$,
where $\Delta(\cA_i)$ is the set of probability distributions over $\cA_i$.
Each player tries to maximize their expected utility given the initial starting history $\emptyset$. We assume finite games, so every history $h$ is bounded in length. 
The expected value of a joint policy $\bpi$ (all players' policies) for player $i$ is defined as
\begin{equation}
v_{i,\bpi} = \bE_{z \sim \bpi}[ u_i(z) ], 
\label{eq:root-value}
\end{equation}
where the terminal histories $z \in \cZ$ are composed of actions drawn from the joint policy.
We also define state-action values for joint policies. The value $q_{i,\bpi}(s,a)$ represents the expected return starting at state $s$, taking action $a$, and playing $\bpi$:
\begin{align}
q_{i,\bpi}(s,a) = & \bE_{z \sim \bpi}[ u_i(z)~|~h \in s, ha \sqsubseteq z] \notag \\
                = & \sum_{h \in s, z \in \cZ(h)} Pr(h | s) u_i(z) \notag\\
                = & \frac{\sum_{h \in s} \eta_{\bpi}(h) q_{i,\bpi}(h, a)}{\sum_{h \in s} \eta_{\bpi}(h)},
\label{eq:q-values}
\end{align}
where $q_{i,\bpi}(h, a) = \bE_{z \sim \bpi}[ u_i(z)~|~ha \sqsubseteq z] = \sum_{h \in s, z \in \cZ(h)}\eta_{\bpi}(h,z) u_i(z)$
is the expected utility of the ground state-action pair $(h,a)$, and $\eta_\bpi(h)$ is the probability of reaching $h$ under the policy $\bpi$.
We make the common assumption that players have \defword{perfect recall}, i.e. they do not forget anything they have observed while playing. Under perfect recall, the distribution of the states can be obtained only
from the opponents' policies using Bayes' rule (see~\cite[Section 3.2]{srinivasan2018actor}).

Each player $i$ tries to find a policy that maximizes their own value $v_{i,\bpi}$. However, this is difficult to do independently since the value depends on the joint policy, not just player $i$'s policy. 
A \defword{best response} policy for player $i$ is defined to be $b_i(\bpi_{-i}) \in \textsc{BR}(\bpi_{-i}) = \{ \bpi_i~|~v_{i, (\bpi_i, \bpi_{-i})} = \max_{\bpi_i'} v_{i,(\bpi_i', \bpi_{-i})} \}$.
Given a joint policy $\bpi$, the \defword{exploitability} of a policy $\bpi_{-i}$ is how much the other player
could gain if they switched to a best response: $\delta_i(\bpi) = \max_{\bpi'_{i}} v_{i, (\bpi'_i, \bpi_{-i})} - v_{i,\bpi}$.
In two-player zero-sum games, an $\epsilon$-minmax (or $\epsilon$-Nash equilibrium) policy is one where $\max_i \delta_i(\bpi) \le \epsilon$. A Nash equilibrium is achieved when $\epsilon = 0$. A common metric to measure the distance to Nash is \textsc{NashConv}$(\bpi) = \sum_i \delta_i(\bpi)$.

\subsection{Extensive-Form Fictitious Play (XFP)}

Extensive-form fictitious play (XFP) is equivalent to standard fictitious play, except that it operates in the extensive-form representation of the game~\cite{Heinrich15FSP}. In fictitious play, the joint policy is initialized arbitrarily (e.g. uniform random distribution at each information state), and players learn by aggregating best response policies.
\begin{algorithm2e}[h!]
\SetKwInOut{Input}{input}\SetKwInOut{Output}{output}
\Input{$\bpi^0$ --- initial joint policy}
\caption{Fictitious Play \label{alg:fp}}
\For{$t \in \{1, 2, \cdots\}$}{
  \For{$i \in \{1, \ldots, n\}$}{
    Compute a best response $\bec{b}_i^t(\bpi_{-i}^{t-1})$ \;
    Update average policy $\bpi^t$ to include $\bec{b}_i^t$ \;
  }
}
\end{algorithm2e}
The extensive-form version, XFP, requires game-tree traversals to compute the best responses and specific update rules that account for the reach probabilities to ensure that the updates are equivalent to the classical algorithm, as described in~\cite[Section 3]{Heinrich15FSP}.
Fictitious play converges to a Nash equilibrium asymptotically in two-player zero-sum games. Sample-based approximations to the best response step have also been developed~\cite{Heinrich15FSP} as well as function approximation methods to both steps~\cite{Heinrich16}. Both steps have also been generalized to other best response
algorithms and meta-strategy combinations~\cite{Lanctot17PSRO}.

\subsection{Counterfactual Regret Minimization (CFR)}

CFR decomposes the full regret computation over the tree into per information-state regret tables and updates~\cite{CFR}.
Each iteration traverses the tree to compute the local values and regrets, updating cumulative
regret and average policy tables, using a local regret minimizer to derive the current policies
at each information state.

The quantities of interest are \defword{counterfactual values}, which are similar to $Q$-values, but
differ in that they weigh only the opponent's reach probabilities, and are not normalized. Formally,
let $\eta_{-i,\bpi}(h)$ be {\it only the opponents' contributions} to the probability of reaching $h$
under $\bpi$.
Then, similarly to equation \ref{eq:q-values}, we define counterfactual values: $q^c_{i,\bpi}(s, a) = \sum_{h \in s} \eta_{-i,\bpi}(h) q_{i,\bpi}(h, a)$, and $v^c_{i,\bpi}(s) = \sum_{a \in \cA_i} \pi_i(s,a) q^c_{i,\bpi}(s, a)$.
On each iteration $k$, with a joint policy $\bpi^k$,
CFR computes a counterfactual regret $r(s, a) = q^c_{i,\bpi^k}(s, a) - v^c_{i,\bpi^k}(s)$
for all information states $s$, and a new policy from the cumulative regrets of $(s,a)$ over the iterations 
using regret-matching~\cite{Hart00}. The average policies converge to an $\epsilon$-Nash equilibrium in $O(|\cS_i|^2 |\Actions_i| / \epsilon^2)$ iterations.

\subsubsection{CFR Versus a Best Response Oracle (CFR-BR)}

Instead of both players employing CFR (CFR-vs-CFR), each player can use CFR versus their worst-case
(best response) opponent, i.e. simultaneously running CFR-vs-BR and BR-vs-CFR. This is the main idea behind
counterfactual regret minimization against a best response (CFR-BR) algorithm~\cite{Johanson12CFRBR}.
The combined average policies of the CFR players is also guaranteed to converge to an $\epsilon$-Nash
equilibrium. In fact, the current strategies also converge with high probability. Our convergence analyses
are based on CFR-BR, showing that a policy gradient versus a best responder also converges
to an $\epsilon$-Nash equilibrium.

\subsection{Policy Gradients in Games}

We consider policies $\bm{\pi}_{\bm{\theta}} = (\pi_{i, \theta_i})_i$ each policy are parameterized by a vector of parameter $(\theta_i)_i = \bm{\theta}$. Using the likelihood ratio method, the gradient of $v_{i, \bm{\pi_{\theta}}}$ with respect to the vector of parameters $\theta_i$ is:

\begin{align}
\grad_{\theta_i} v_{i,\bm{\pi_\theta}} &= \sum \limits_{s \in S_i} \left(\sum \limits_{h \in s} \eta_{\bm{\pi}}(h) \right)  \sum_{a} \grad_{\theta_i} \pi_{i, \theta_i}(s, a)  q_{i, \bm{\pi_\theta}}(s, a)
\label{eq:policy.gradient.iig}
\end{align}
This result can be seen as an extension of the policy gradient Theorem~\cite{Sutton00policy,glynn1995likelihood,williams1992simple,baxter2001infinite}
to imperfect information games and has been used under several forms:
for a detailed derivation, see~\cite[Appendix D]{srinivasan2018actor}.

The critic ($q_{i, \bm{\pi_\theta}}$) can be estimated in many ways (Monte Carlo Return~\cite{williams1992simple} or using a critic for instance in~\cite{srinivasan2018actor} in the context of games. Then:
\beqan
\theta_i &\leftarrow &\theta_i + \alpha \sum_{l=0}^K   \mathds{1}_{i=\tau(s_l)} \sum_{a} \grad_{\theta_i} \pi_{i, \theta_i}(s_l, a)  \hat q_{i, \bm{\pi_\theta}}(s_l, a),
\eeqan
where $\alpha$ is the learning rate used by the algorithm and $\hat q_{i, \bm{\pi_\theta}}(s_l, a)$ is the estimation of the return used.

\section{Exploitability Descent \label{sec:ed}}

Exploitability Descent~(ED) follows the basic form of the classic convex-concave optimization problem for solving matrix games~\cite{GTK51,Boyd04}.
Conceptually, the algorithm is uncomplicated and shares the outline of fictitious play: on each iteration, there are two steps that occur for each player. The first step is identical to fictitious play: compute the best response to each player's policy. The second step then performs gradient ascent on the policy to increase each player's utility against the respective best responder (aiming to decrease each player's exploitability). 
\begin{algorithm2e}[h!]
\SetKwInOut{Input}{input}\SetKwInOut{Output}{output}
\Input{$\bpi^0$ --- initial joint policy}
\caption{Exploitability Descent (ED) \label{alg:ed}}
\For{$t \in \{1, 2, \cdots\}$}{
  \For{$i \in \{1, \cdots, n\}$}{
    Compute a best response $\bec{b}_i^t(\bpi_{-i}^{t-1})$ \;
  }
  \For{$i \in \{1, \cdots, n\}, s \in \cS_i$}{
      Define $\bec{b}^t_{-i} = \{ \bec{b}^t_j \}_{j \not= i}$ \;
      Let $\bq^{\bec{b}}(s) = \textsc{ValuesVsBRs}(\bpi^{t-1}_i(s), \bec{b}^t_{-i})$ \label{alg:ed-values}\;
      $\bpi^t_i(s) = \textsc{GradAscent}(\bpi^{t-1}_i(s), \alpha^t, \bq^{\bec{b}}(s))$\label{alg:ed-update}\;
  }
}
\end{algorithm2e}
The change in the second step is important for two reasons. First, it leads to a convergence of the policies that are being optimized without having to compute an explicit average policy, which is complex in the sequential setting. Secondly, the policies can now be easily parameterized (i.e. using e.g. deep neural networks) and trained using policy gradient ascent without storing a large buffer of previous data. 

The general algorithm is outlined in Algorithm~\ref{alg:ed}, where $\alpha^t$ the learning rate on iteration $t$.
Two steps (lines \ref{alg:ed-values} and \ref{alg:ed-update}) are intentionally unspecified: we will show properties for two specific instantiations of this general ED algorithm.
The quantity $\bq^{\bec{b}}$ refers to a set of expected values for player $i = \tau(s)$, one for each action at $s$ using $\bpi^{t-1}_i$ against a set of individual best responses.
The \textsc{GradientAscent} update step unspecified for now as we will describe several forms, but the main idea is to increase/decrease the probability of higher/lower utility actions via the gradients of the value functions, and project back to the space of policies.

\subsection{Tabular ED with $q$-Values and $\ell_2$ Projection}
\label{sec:edq}

For a vector of $|\cA|$ real numbers $\btheta_s$, define the \defword{simplex} as $\Delta_s = \{ \btheta_{s,a}~|~\btheta_s \ge \mathbf{0}, \sum_a \theta_{s,a} = 1 \}$, and the $\ell_2$ projection as
$\Pi_{\ell_2}(\btheta_s) = \argmin_{\btheta_s' \in \Delta_s}||\btheta_s' - \btheta_s||_2$.

Let $\bpi_{\btheta}$ be a joint policy parameterized by $\btheta$, and $\bpi_{\btheta_i}$ refer to the
portion of player $i$'s parameters (\ie in tabular form $\{ \btheta_s~|~\tau(s) = i \}$). Here each parameter is a probability of an action at a particular state: $\theta_{s,a} = \pi_{\btheta}(s,a)$.
We refer to TabularED($q$, $\ell_2$) as an instance of exploitability descent with
\begin{equation}
\bq^{\bec{b}}(s) = \{ q_{i,(\bpi^{t-1}_{\theta}, \bec{b}^t_{-i})}(s, a) \}_{a \in \cA},
\label{eq:edq-values}
\end{equation}
and the policy gradient ascent update defined to be
\begin{eqnarray}
\theta^t_s & = & \Pi_{\ell_2} \left(\btheta^{t-1}_s + \alpha^t \ip{\nabla_{\btheta_s} \bpi^{t-1}_{\btheta}(s)}{\bq^{\bec{b}}(s)}\right) \notag\\
           & = & \Pi_{\ell_2} \left(\btheta^{t-1}_s + \alpha^t \bq^{\bec{b}}(s)\right),  \label{eq:edq-update}
\end{eqnarray}
where the Jacobian $\nabla_{\btheta_s} \bpi^{t-1}_{\btheta}(s)$ is an identity matrix because each parameter $\theta_{s,a}$ corresponds directly to the probability $\pi_{\btheta}(s,a)$, and $\ip{\cdot}{\cdot}$ is the usual matrix inner product.

\subsection{Tabular ED with Counterfactual Values and Softmax Transfer Function}
\label{sec:edcfv}

For some vector of real numbers, ${\btheta_s}$, define softmax$(\btheta_s) = \{ \softmaxTransfer(\btheta_s) \}_a = \{ \exp(\theta_{s,a}) / \sum_{a'} \exp(\theta_{s,a'}) \}_a$.
Re-using the tabular policy notation from the previous section, we now define a different instance of
exploitability descent. We refer to TabularED($q^c$, softmax) as the algorithm that specifies
$\bpi_{\btheta}(s) = \softmaxTransfer(\btheta_s)$,
\begin{equation}
\bq^{\bec{b}}(s) = \{ q^c_{i,\bpi}((\bpi^{t-1}_{\btheta}, \bec{b}^t_{-i}), s, a) \}_{a \in \cA},
\label{eq:edcfv-values}
\end{equation}
and the policy parameter update as
\begin{equation}
\btheta^t_s = \btheta^{t-1}_s + \alpha^t \ip{\nabla_{\btheta_s} \bpi^{t-1}_{\btheta}(s)}{\bq^{\bec{b}}(s)},
\label{eq:edcfv-update}
\end{equation}
where $\nabla_{\btheta_s} \bpi^{t-1}_{\btheta}(s)$ represents the Jacobian of softmax.

\subsection{Convergence Analyses}

We now analyze the convergence guarantees of ED. We give results for two cases: first, in cyclical perfect information games and 
Markov games, and secondly imperfect information games. All the proofs are found in the Appendix~\ref{app:proofs}.

\subsubsection{Cyclical Perfect Information Games and Markov Games}
The following result extends the policy gradient theorem~\cite{Sutton00policy,glynn1995likelihood,williams1992simple,baxter2001infinite} to the zero-sum two-player case. It proves that a generalized gradient of the worst-case value function can be estimated from experience as in the single player case.

\begin{theorem}[Policy Gradient in the Worst Case]
    \label{thm:pg-worst-case}
    The gradient of policy $\bec{\pi_{\theta_i}}$'s value, $v_{i,(\bec{\pi_{\theta_i}}, \bec{b})}$, against a best response,
    $\bec{\beta} \as \bec{b_{-i}(\pi_{\theta_i})} \in \BrSet{\pi_{\theta_i}}$ is a generalized gradient (see~\cite{clarke1975generalized}) of $\bec{\pi_{\theta_i}}$'s worst-case value function,
    \begin{align}
        & \grad_{\theta_i} v_{i,(\pi_{\theta_i}, b_{-i}(\pi_{\theta_i}))} \in \subdiff \min_{\pi_{-i}} v_{i,(\pi_{\theta_i}, \pi_{-i})}.\nonumber
    \end{align}

\end{theorem}

All of the proofs are found in Appendix~\ref{app:proofs} of the technical report version of the paper~\cite{ED-arXiv}.

This theorem is a natural extension of the policy gradient theorem to the zero-sum two-player case. As in policy gradient, this process is only guaranteed to converge to a local maximum of the worst case value $\min_{\pi_{-i}} v_{i,(\pi_{\theta_i}, \pi_{-i})}$ of the game but not necessarily to an equilibrium of the game. An equilibrium of the game is reached when the two following conditions are met simultaneously: (1) if the policy is tabular and (2) if all states are visited with at least some probability for all policies. This statement is proven in Appendix~\ref{GMC}.

The method is called exploitability descent because policy gradient in the worst case minimizes exploitability. In a two-player, zero-sum game, if both players independently run ED, $\NashConv$ is locally minimized.

\begin{lemma}
    \label{lem:sim-pg-min-nashconv}
    In the two-player zero-sum case, simultaneous policy gradient in the worst case locally minimizes $\NashConv$.
\end{lemma}

\subsubsection{Imperfect Information Games \label{sec:conv-iig}}

We now examine convergence guarantees in the imperfect information setting. 
There have two main techniques used to solve adversarial games in this case: the first is to rely on the sequence-form representation of policies which makes the optimization problem convex~\cite{SequenceFormLPs,EGT-OO}.
The second is to weight the values by the appropriate reach probabilities, and employ local optimizers~\cite{CFR,Johanson12CFRBR}.
Both take into account the probability of reaching information states,
but the latter allows a convenient tabular policy representation.

We prove finite time exploitability bounds for TabularED($q$, $\ell_2$), and we relate TabularED($q^c$, softmax) to a similar algorithm that also has finite time bounds.

The convergence analysis is built upon two previous results: the first is CFR-BR~\cite{Johanson12CFRBR}.
The second is a very recent result that relates policy gradient optimization in imperfect information games to CFR~\cite{srinivasan2018actor}. 
The result here is also closely related to the optimization against a worst-case opponent~\cite[Theorem 4]{Waugh14Unified}, except our policies are expressed in tabular (\ie behavioral) form rather than the sequence form.\\

\noindent {\bf \underline{Case}: TabularED($q$, $\ell_2$)}.
Recall that the parameters $\btheta = \{ \theta_{s,a} \}_{s \in \cS, a \in \cA}$ correspond to the tabular policy. For convenience, let $\btheta_s = \{ \theta_{s,a}\}_{a \in \cA}$.

We now present the main theorem\footnote{Note that this theorem is different than the one presented in the original version of this work. Please see Appendix~\ref{app:errata} for a discussion of the differences.}, which states that if both players optimize their policies using TabularED($q$, $\ell_2$), it will generate policies with decreasing regret, which combined form an approximate Nash equilibrium.
\begin{theorem}
Let TabularED($q$, $\ell_2$) be described as in Section~\ref{sec:edq} using tabular policies and the update rule in Definition~\ref{def:spgpi-sep}. 
In a two-player zero-sum game, if each player updates their policy simultaneously using TabularED($q$, $\ell_2$),
if $\forall s, a \in \cS \times \cA: \theta_{s,a} > 0$ and $\alpha^t = t^{-\frac{1}{2}}$,
then for each player $i$: after $T$ iterations, a policy $\bpi_i^* \in \{ \bpi_i^1, \cdots, \bpi_i^T \}$ will have been generated such that $\bpi_i^*$ is $i$'s part of a $\frac{2\epsilon}{T}$-Nash equilibrium, where $\epsilon = |\cS_i| \left( \sqrt{T} + (\sqrt{T} - \frac{1}{2}) |\cA_i| (\Delta_{u_i})^2 \right) + O(T)$, and 
$\Delta_u = \max_{z, z' \in \cZ}(u_i(z) - u_i(z'))$.
\label{thm:edq}
\end{theorem}

ED is computing best responses each round already, so it is easy to track the best iterate: it will simply
be the one with the highest expected value versus the opponent's best response. 
However, because of the $O(T)$ term, TabularED($q$, $\ell_2$) may not converge to an equilibrium,
whereas this is guaranteed when using counterfactual values $q^c$ (see below).

The reasoning behind the proof can also be applied to the original CFR-BR theorem, so we present an improved
guarantee, whereas the original CFR-BR theorem made a probabilistic guarantee.
\begin{corollary}{(Improved \cite[Theorem 4]{Johanson12CFRBR})}
\label{cor:cfr-br}
If player $i$ plays $T$ iterations of CFR-BR, then it will have generated a $\bpi^*_i \in \{ \bpi^1, \bpi^2, \cdots, \bpi^T\}$, where $\bpi^*_i$ is a $2\epsilon$-equilibrium, where $\epsilon$ is defined as in \cite[Theorem 3]{Johanson12CFRBR}.
\end{corollary}

The best iterate can be tracked in the same way as ED, and the convergence is guaranteed.

\begin{remark}
When using $q$-values, the values are normalized by a quantity, $\cB_{-i}(\bpi, s)$, that depends on the opponents' policies~\cite[Section 3.2]{srinivasan2018actor}.
Therefore, the regret bound is undefined when $\cB_{-i}(\bpi,s) = 0$, which can happen when an opponent no longer plays to reach $s$.\\
\label{rem:qvals}
\end{remark}

\noindent {\bf \underline{Case}: TabularED}($q^c$, $\ell_2$).
Instead of using q-values, we can implement ED with counterfactual values. In this case, TabularED with the $\ell_2$ projection becomes CFR-BR(GIGA), which provably converges to approximate equilibria and avoids the issued discussed in Remark~\ref{rem:qvals}.

\begin{theorem}
    \label{thm:qc-l2}
    Let TabularED($q^c$, $\ell_2$) be described as in Section~\ref{sec:edq} using tabular policies and the following update rule:
    \begin{align*}
        \bpi_i^t(s) = \projection_{\ell_2} \subex{ \bpi_i^{t - 1}(s) + \stepSize^t \bec{q}^{c, \bec{b}}(s) }.
    \end{align*}
    Then, the Theorem~\ref{thm:edq} statement, excluding the $O(T)$ term, also holds for TabularED($q^c$, $\ell_2$) .\\
\end{theorem}

\noindent {\bf \underline{Case}: TabularED($q^c$, softmax)}
We now relate TabularED with counterfactual values and softmax policies closely to an algorithm with known finite time convergence bounds. For details, see Appendix~\ref{sec:connections}.

TabularED($q^c$, softmax) is still a policy gradient algorithm: it differentiates the policy (\ie softmax function) with respect to its parameters, and updates in the direction of higher value. With two subtle changes to the overall process, we can show that the algorithm would become CFR-BR using hedge~\cite{hedge} as a local regret minimizer. CFR with hedge is known to have a better bound, but has typically not performed as well as regret matching in practice, though it has been shown to work better when combined with pruning based on dynamic probability thresholding~\cite{Brown17Dynamic}.

Instead of policy gradient, one can use a softmax transfer function over the the sum of action values (or regrets) over time, which are the gradients of the {\it value function} with respect to the policy. Accumulating the gradients in this way, the algorithm can be recognized as Mirror Descent~\cite{Nemirovsky1983mirrorDescent}, which also coincides with hedge given the softmax transfer~\cite{Beck2003entropicDescent}. When using the counterfactual values, ED then turns into CFR-BR(hedge), which converges for the same reasons as CFR-BR(regret-matching).


We do not have a finite time bound of the exploitability of TabularED($q^c$, softmax) as we do for the same algorithm with an $\ell_2$ projection or CFR-BR(hedge). But since TabularED($q^c$, softmax) is a policy gradient algorithm, its policy will be adjusted toward a local optimum upon each update and will converge at that point when the gradient is zero. We use this algorithm because the policy gradient formulation allows for easily-applicable general function approximation.

\section{Experimental Results}
We now present our experimental results. We start by comparing empirical convergence rates to XFP and CFR in the tabular
setting, following by convergence behavior when training neural network functions to approximate the policy.

In our initial experiments, we found that using $q$-values led to plateaus in convergence in some cases, possibly due to numerical instability caused by the problem outlined in Remark~\ref{rem:qvals}.
Therefore, we present results only using TabularED($q^c$, softmax), which for simplicity we refer to as TabularED for the remainder of this section. We also found that the algorithm converged faster with slightly higher learning rates than the ones suggested by Section~\ref{sec:conv-iig}.

\subsection{Experiment Domains}

Our experiments are run across four different imperfect information games. We provide very brief descriptions here; see Appendix~\ref{app:domains} as well as~\cite{Kuhn50,Southey05uai} and ~\cite[Chapter 3]{lanctot13phdthesis} for more detail. 

{\bf Kuhn poker} is a simplified poker game first proposed by Harold Kuhn~\cite{Kuhn50}
{\bf Leduc poker} is significantly larger game with two rounds and a 6-card deck in two suits, e.g. \{JS,QS,KS, JH,QH,KH\}.
{\bf Liar's Dice}(1,1) is dice game where each player gets a single private die, rolled at the start of the game, and players proceed to bid on the outcomes of all dice in the game.
{\bf Goofspiel} is a card game where players try to obtain point cards by bidding simultaneously. We use an imperfect information variant where bid cards are unrevealed.

\subsection{Convergence Results \label{sec:conv-results}}

We now present empirical convergence rates to $\epsilon$-Nash equilibria.
The main results are depicted in Figure~\ref{fig:all_plots_all_games}. 
\begin{figure}[h]
\centering
\begin{tabular}{@{}c@{}}
    \includegraphics[width=.49\textwidth]{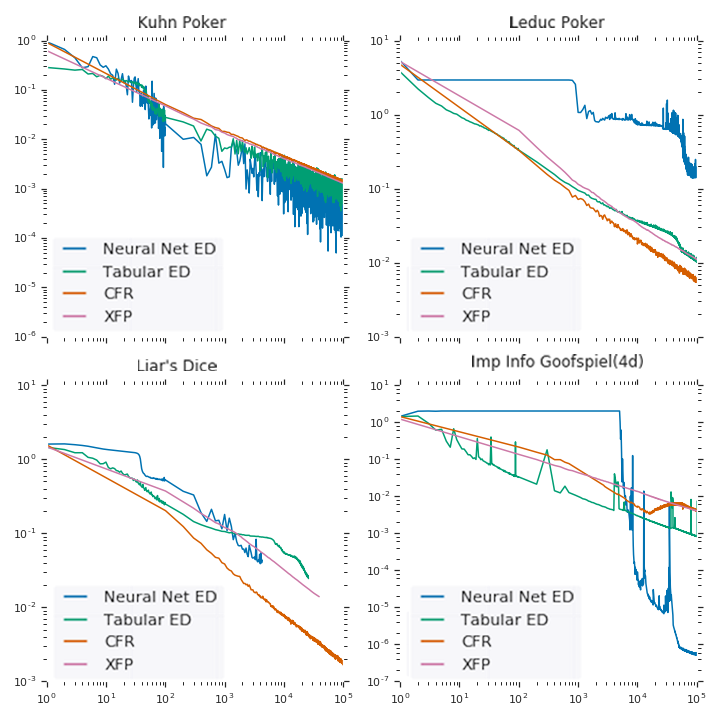} 
  \end{tabular}
  \caption{Extensive-form fictitious play (XFP), CFR, tabular and neural-net ED. The y-axis is $\textsc{NashConv}$ defined in Section~\ref{sec:background}, and the x-axis is number of iterations, both in log-scale.}
  \label{fig:all_plots_all_games}
\end{figure}

For the neural network experiments, we use a single policy network for both players, which takes as input the current state of the game and whose output is a softmax distribution over the actions of the game.
The state of the game is represented in a game-dependent fashion as a fixed-size vector of between 11 and 52 binary bits, encoding public information, private information, and the game history.

The neural network consists of a number of fully-connected hidden layers, each with the same number of units and with rectified linear activation functions after each layer. A linear output layer maps from the final hidden layer to a value per action. The values for the legal actions are selected and mapped to a policy using the softmax function.

At each step, we evaluate the policy for every state of the game, compute a best response to it, and evaluate each action against the best response. We then perform a single gradient descent step on the loss function: $- \sum_s \bpi_i(s) \cdot \left( \bq^b(s) - B(s) \right) + w_r \frac{1}{n} \sum_i \btheta_i^2$,
where the final term is a regularization for all the neural network weights, and the baseline $B(s)$ is a computed constant (i.e. it does not contribute to the gradient calculation) with $B(s) = \bpi_i(s) \cdot \bq^{\bec{b}}(s)$.
We performed a sweep over the number of hidden layers (from 1 to 5), the number of hidden units (64, 128 or 256), the regularization weight ($10^{-7}, 10^{-6}, 10^{-5}, 10^{-4}$), and the initial learning rate (powers of 2). The plotted results show the best values from this sweep for each game.

\subsection{Discussion}

There are several interesting observations to make about the results. 
First, the convergence of the neural network policies is more erratic than the tabular counterparts.
However, in two games the neural network policies have learned {\it more accurate} approximate equilibria than any of the tabular algorithms for the same number of iterations.
The network could be generalizing across the state space (discovering patterns)
in a way that is not possible in the tabular case, despite raw input features.

Although Tabular ED and XFP have roughly the same convergence rate, the respective function approximation versions have an order of magnitude difference in speed, with Neural ED reaching an exploitability of $0.08$ in Leduc Poker after $10^5$ iterations, a level which NFSP reaches after approximately $10^6$ iterations~\cite{Heinrich16}. Neural ED and NFSP are not directly comparable as NFSP is computing an approximate equilibrium using sampling and RL while ED uses true best response. However, NFSP uses a reservoir buffer dataset of 2 million entries, whereas this is not required in ED.

\section{Conclusion}

We introduce Exploitability Descent (ED) that optimizes its policy directly against worst-case opponents. In cyclical perfect information and Markov games, we prove that ED policies converge to strong policies that are unexploitable in the tabular case. In imperfect information games, we also present finite time exploitability bounds for tabular policies.
While the empirical convergence rates using tabular policies are comparable to previous algorithms,
the policies themselves provably converge. So, unlike XFP and CFR, there is no need to compute the average policy.
Neural network function approximation is applicable via direct policy gradient ascent, also avoiding domain-specific abstractions, 
or the need to store large replay buffers of past experience, as in neural fictitious self-play~\cite{Heinrich16}, or a set of past
networks, as in PSRO~\cite{Lanctot17PSRO}.

In some of our experiments, neural networks learned lower-exploitability policies than the tabular counterparts, which could be an
indication of strong generalization potential by recognizing similar patterns across states.
There are interesting directions for future work: using approximate best responses and sampling trajectories for the policy
optimization in larger games where enumerating the trajectories is not feasible.

\section*{Acknowledgments}

We would like to thank Neil Burch, Johannes Heinrich, and Martin Schmid for feedback. Dustin Morrill was supported by The Alberta Machine Intelligence Institute (Amii) and Alberta Treasury Branch (ATB) during the course of this research.

\bibliographystyle{named}
\bibliography{paper}
 

\vspace{0.2cm}
\appendix
{\Large {\bf Appendices}}

\section{Proofs}
\label{app:proofs}

\subsection{Cyclical Perfect Information Games and Markov Games}

\begin{proof}[Proof of Theorem~\ref{thm:pg-worst-case}]
    The proof uses tools from the non-smooth analysis to properly handle gradients of a non-smooth function. We use the notion of generalized gradients defined in~\cite{clarke1975generalized}. The generalized gradient of a Lipschitz function is the convex hull of the limits of the form $\lim \nabla f(x + h_i)$ where $h_i \rightarrow 0$. The only assumption we will require is that the parameters of our policy $\theta_i$ remains in a compact set $\Theta_i$ and that $v_{i,(\pi_{\theta_i}, \pi_{-i})}$ is differentiable with respect to $\theta_i$ for all $\pi_{-i}$.
    
    More precisely we use~\cite[Theorem 2.1]{clarke1975generalized} to state our result. The theorem requires the function to be uniformly semi-continuous which is the case if the policy is differentiable since the dependence of $v_{i,(\pi_{\theta_i}, \pi_{-i})}$ on $\pi_{-i}$ is polynomial. The function $v_{i,(\pi_{\theta_i}, \pi_{-i})}$ is Lipschitz with respect to $(\theta_i, \pi_{-i})$. The uniform continuity of $\grad_{\theta_i} v_{i,(\pi_{\theta_i}, \pi_{-i})}$ comes from the fact that $\Theta_i$ is compact.
    
    Using~\cite[Theorem 2.1]{clarke1975generalized}, we have that $\subdiff \min_{\pi_{-i}} v_{i,(\pi_{\theta_i}, \pi_{-i})}$ is the convex hull of 
    $
    \{ \grad_{\theta_i} v_{i,\pi_{\theta_i}, \beta} | \beta \in \BrSet{\pi_{\theta_i}} ) \}
    $, so 
    $
    \grad_{\theta_i} v_{i,\pi_{\theta_i}, \beta} \in \subdiff \min_{\pi_{-i}} v_{i,(\pi_{\theta_i}, \pi_{-i})}
    $.
    
    The proof follows by applying the policy gradient theorem~\cite{baxter2001infinite} to $\grad_{\theta_i} v_{i,\pi_{\theta_i}, \beta}$.
\end{proof}

\begin{proof}[Proof of Lemma~\ref{lem:sim-pg-min-nashconv}]
    In a two-player, zero-sum game, $\NashConv$ reduces to the sum of exploitabilities:
    \begin{align}
        & \NashConv(\pi_{\theta_i}, \pi_{\theta_{-i}}) = \sum_i \max_{\pi_{i}'} v_{i,(\pi_i', \pi_{\theta_{-i}})} - v_{i,\bm\pi}\nonumber\\
        &= \sum_i \max_{\pi_{i}'} v_{i,(\pi_i', \pi_{\theta_{-i}})} = \sum_i -\min_{\pi_{i}'} -v_{i,(\pi_i', \pi_{\theta_{-i}})} \nonumber\\
        &= -(\min_{\pi_{1}'} v_{2,(\pi_{\theta_{2}}, \pi_1')} + \min_{\pi_{2}'} v_{1,(\pi_{\theta_{1}}, \pi_2')})\nonumber
    \end{align}
    so doing policy gradient in the worst case independently for all players locally minimizes the sum of exploitabilities and therefore $\NashConv$. Formally we have\footnote{Usually one would have $\subdiff_{\bm{\theta}}(f_1 + f_2) \subset \subdiff_{\bm{\theta}} f_1 + \subdiff_{\bm{\theta}} f_2$. But since in our case the functions are defined on two different sets of parameters, we have an equality.}:
    
    \begin{align}
        &\subdiff_{\bm{\theta}} \sum_i \max_{\pi'_{i}} v_{i,(\pi^{'i}, \pi^{-i}_{\theta_i})} - v_{i,\bm\pi}\\ \nonumber
        &= - \left(\subdiff_{\theta_2} \min_{\pi_{1}'} v_{2,(\pi_{\theta_{2}}, \pi_1')} + \subdiff_{\theta_2} \min_{\pi_{2}'} v_{1,(\pi_{\theta_{1}}, \pi_2')}\right)\nonumber.
    \end{align}
\end{proof}

\subsection{Imperfect Information Games}

Before we prove the main proof (of Theorem~\ref{thm:edq}), we start with some definitions, a previous theorem, and a supporting lemma.

\begin{definition}
Recall $\Delta_s$ and $\Pi_{\ell_2}$ from Section~\ref{sec:edq}.
Define \defword{strong policy gradient policy iteration against specific opponent(s)}, SPGPI$(\bpi_{-i})$, to be the process where player $i$ updates their policy using state-localized gradient ascent while all other players use $\bpi_{-i}$, according to Eq~\ref{eq:edq-update}.
\label{def:spgpi-sep}
\end{definition}
This is a form of strong policy gradient policy iteration (SPGPI) defined in~\cite[Theorem 2]{srinivasan2018actor} that separates the optimization for each player. Also notice that an iteration of TabularED($q$, $\ell_2$) is equivalent to $n$
simultaneous applications of SPGPI$(b_{-i}(\bpi_i))$.
\begin{definition}
Suppose all players use a sequence of joint policies $\{\bpi^1, \cdots, \bpi^T\}$ over $T$ iterations.
Define player $i$'s \defword{regret} after $T$ iterations to be the difference in expected utility between the best possible policy in hindsight and the expected utility given the sequence of policies:
\[ R^T_i = \max_{\bpi_i} \sum_{t=1}^T \left( v_{i,(\bpi_i, \bpi^t_{-i})} - v_{i,\bpi^t} \right). \]
\end{definition}

\begin{theorem}{\cite[Theorem 2]{srinivasan2018actor}}
Suppose players play a finite game using joint policies $\{\bpi^1, \cdots, \bpi^T\}$ over $T$ iterations.
In a two-player zero-sum game, if $\forall s, a \in \cS \times \cA: \theta_{s,a} > 0$ and $\alpha^t = t^{-\frac{1}{2}}$, then the regret of SPGPI$(\bpi_{-i})$ after $T$ iterations is 
$R^T_i \le \epsilon$, where $\epsilon = |\cS_i| \left( \sqrt{T} + (\sqrt{T} - \frac{1}{2}) |\cA_i| (\Delta_{u_i})^2 \right) + O(T)$, and 
$\Delta_u = \max_{z, z' \in \cZ}(u_i(z) - u_i(z'))$.
\label{thm:spgpi-br}
\end{theorem}
Note that, despite the original application of policy gradients in self-play, it follows from the original proof that
the statement about the regret incurred by player $i$ does not require a specific algorithm generate the opponents' policies: it is only a function of the specific sequence of opponent policies.
In particular, they could be best response policies, and so SPGPI$(b_{-i}(\bpi_i))$ has the same regret guarantee.

The following lemma states an optimality bound of the best iterate under time-independent loss functions equal to the average regret. The best strategy in a no-regret sequence of strategies then approaches an equilibrium strategy over time without averaging and without probabilistic arguments.

\begin{lemma}\label{lm:cfr-br-opt}
Denote $[[T]] = \{ 1, 2, \cdots, T \}$. 
For any sequence of $T$ iterates, $\{ \btheta^t \in \Theta \}_{t \in [[T]]}$, from decision set $\Theta$, the regret of this sequence under loss, $\ell(\btheta) : \Theta \rightarrow \Re$, is 
\[
R^T = \sum_{t = 1}^T \ell(\btheta^t) - \inf_{\btheta^* \in \Theta} \sum_{t=1}^T \ell(\btheta^*).
\]
Then, the iterates with the lowest loss, $\hat{\btheta} = \argmin_{t \in [[T]]} \ell(\btheta^t)$,
has an optimality gap bounded by the average regret:
\[
\ell(\hat{\btheta}) - \inf_{\btheta^* \in \Theta} \ell(\btheta^*) \le \frac{R^T}{T}.
\]
\end{lemma}
\begin{proof}
Since $\ell$ is fixed (not varying with $t$),
\begin{eqnarray*}
R^T & =   & \sum_{t=1}^T \ell(\btheta^t) - T \inf_{\btheta^* \in \Theta} \ell(\btheta^*)\\
    & \ge & T ( \min_{t \in [[T]]} \ell(\btheta^t) - \inf_{\btheta^* \in \Theta} \ell(\btheta^*) )\\
    & =   & T ( \ell(\hat{\btheta}) - \inf_{\btheta^* \in \Theta} \ell(\btheta^*) ),
\end{eqnarray*}
and dividing by $T$ yields the result.
\end{proof}

In finite games, we can replace the $\inf$ operation in Lemma~\ref{lm:cfr-br-opt} with $\min$ when the decision set is the set of all possible strategies, since this set is closed.

We now provide the proof of the main theorem.

\begin{proof}[Proof of Theorem~\ref{thm:edq}]
The first part of the proof follows the logic of the CFR-BR proofs~\cite{Johanson12CFRBR}.
Unlike CFR-BR, we then use Lemma~\ref{lm:cfr-br-opt} to bound the quality of the best iterate.

SPGPI($b_{-i}(\pi_i)$), has regret for player $i$ as described in Theorem~\ref{thm:spgpi-br}.
Define loss function, $\ell_i(\bpi_i) = -\min_{\bpi_{-i}} v_{i, (\bpi_i, \bpi_{-i})}$, as the negated worst-case value for player $i$, like that described by~\cite[Theorem 4]{Waugh14Unified}. Then by Lemma~\ref{lm:cfr-br-opt} and Theorem~\ref{thm:spgpi-br} we have, for the best iterate:
\begin{eqnarray}
\ell_i(\pi^*) - \min_{\pi_i'} \ell_i(\pi_i', \pi_{-i}) & \le & \frac{R^T}{T}\\
                                                       & \le & \frac{\epsilon}{T} + O(1),\\
\end{eqnarray}
so $\delta_i(\bpi^*) \le \frac{\epsilon}{T} + O(1)$, 
where $\delta_i(\bpi)$ is the exploitability defined in Section~\ref{sec:background}.
The Nash equilibrium approximation bound is just the sum of the exploitabilities.
\end{proof}

The $O(1)$ term here comes from the $O(T)$ term in Theorem~\ref{thm:spgpi-br}. When counterfactual values are used instead of $q$-values, the bound in Theorem~\ref{thm:spgpi-br} does not contain the $O(T)$ term and hence there is no $O(1)$ term.
As a result Lemma~\ref{lm:cfr-br-opt} can be used to obtain Corollary~\ref{cor:cfr-br}.

\begin{proof}[Proof of Theorem~\ref{thm:qc-l2}]
This update rule is identical to that of generalized infinitesimal gradient ascent (GIGA)~\cite{Zinkevich03Online} at each information state with best response counterfactual values. CFR-BR(GIGA) therefore performs the same updates and the two algorithms coincide.
With step sizes $\alpha^t = t^{-\frac{1}{2}}, 0 < t \leq T$, each local GIGA instance has regret after $T$ iterations upper bounded by $\sqrt{T} + (\sqrt{T} - \frac{1}{2}) |\cA_i| (\Delta_{u_i})^2$,
where $\Delta_u = \max_{z, z' \in \cZ}(u_i(z) - u_i(z'))$.
Then, by the CFR Theorem~\cite{CFR}, the regret is bounded by
\begin{align*}
R^T_i \le |\cS_i| \left( \sqrt{T} + (\sqrt{T} - \frac{1}{2}) |\cA_i| (\Delta_{u_i})^2 \right).
\end{align*}
Specifically, the $O(T)$ term is missing since counterfactual values are used instead of $q$-values, and the
bound is instead inherited from the CFR theorem.
\end{proof}

\section{Longer Descriptions of Experiment Domains}
\label{app:domains}

\begin{description}
\item[Kuhn poker] is a simplified poker game first proposed by Harold Kuhn~\cite{Kuhn50}. Each player antes a single chip, and gets a single private card from a totally-ordered 3-card deck, e.g.$\{J, Q, K\}$. There is a single betting round limited to one raise of 1 chip, and two actions: pass (check/fold) or bet (raise/call).
If a player folds, they lose their commitment (2 if the player made a bet, otherwise 1). If neither player folds, the player with the higher card wins the pot (2, 4, or 6 chips). The utility for each player is defined as the number of chips after playing minus the number of chips before playing.

\item[Leduc poker] is significantly larger game with two rounds and a 6-card deck in two suits, e.g. \{JS,QS,KS, JH,QH,KH\}. Like Kuhn, each player initially antes a single chip to play and obtains a single private card and there are three actions: fold, call, raise. There is a fixed bet amount of 2 chips in the first round and 4 chips in the second round, and a limit of two raises per round. After the first round, a single public card is revealed. A pair is the best hand, otherwise hands are ordered by their high card (suit is irrelevant). Utilities are defined similarly to Kuhn poker.

\item[Liar's Dice(1,1)] is dice game where each player gets a single private die in $\{ \epsdice{1}, \epsdice{2}, \cdots, \epsdice{6} \}$, 
rolled at the beginning of the game. The players then take turns
bidding on the outcomes of both dice, i.e. with bids of the form $q$-$f$ referring to quantity and face, or calling ``Liar''. The bids represent a claim that
there are at least $q$ dice with face value $f$ among both players. The highest die value, $\epsdice{6}$, counts as a wild card matching any
value. Calling ``Liar'' ends the game, then both players reveal their dice. If the last bid is not satisfied, then the player who called ``Liar'' wins.
Otherwise, the other player wins. The winner receives +1 and loser -1.

\item[Goofspiel] or the Game of Pure Strategy, is a bidding card game where players are trying to obtain the most points.
shuffled and set face-down. Each turn, the top point card is revealed, and players simultaneously play a bid card; the point
card is given to the highest bidder or discarded if the bids are equal. In this implementation, we use a fixed deck of decreasing
points. In this paper, we use $K=4$ and an imperfect information variant where players are only told whether they have won or
lost the bid, but not what the other player played.
\end{description}

\section{Connections and Differences Between Gradient Descent, Mirror Descent, Policy Gradient, and Hedge \label{sec:connections}}

There is a broad class of algorithms that attempt to make incremental progress on an optimization task by moving parameters in the direction of a gradient. This elementary idea is intuitive, requiring only basic knowledge of calculus and functions to understand in abstract. One way to more formally justify this procedure comes from the field of online convex optimization~\cite{Zinkevich03Online,Hazan15OCO,Shwartz2012oco}. The linearization trick~\cite{Shwartz2012oco} reveals how simple parameter adjustments based on gradients can be used to optimize complicated non-linear functions. Perhaps the most well known learning rule is that of \defword{gradient descent}: $\bec{\prediction} \leftarrow \bec{\prediction} - \stepSize \grad_{\bec{\prediction}} f(\bec{\prediction})$, where the goal is to minimize function $f$.

Often problems will include constraints on $\bec{\prediction}$, such as the probability simplex constraint required of decision policies. One often convenient way to approach this problem is transform unconstrained parameters, $\bec{\gaVariable}$, to the nearest point in the feasible set, $\bec{\prediction} \as \projection(\bec{\gaVariable})$, with a transfer function, $\projection$. This separation between the unconstrained and constrained space produces some ambiguity in the way optimization algorithms are adapted to handle constraints. Do we adjust the transformed parameters or the unconstrained parameters with the gradient? And do we take the gradient with respect to the transformed parameters or the unconstrained parameters?

\defword{Projected gradient descent (PGD)}~\cite{Zinkevich03Online} resolves these ambiguities by adjusting the transformed parameters with the gradient of the transformed parameters. For PGD, the unconstrained parameters are not saved, they are only produced temporarily before they can be projected into the feasible set, $\bec{\prediction} \leftarrow \projection(\bec{\prediction} - \stepSize \grad_{\bec{\prediction}} f(\bec{\prediction}))$. \defword{Mirror descent (MD)}~\cite{Nemirovsky1983mirrorDescent,Beck2003entropicDescent}, broadly, makes adjustments exclusively in the unconstrained space, and transfers to the feasible set on-demand. However, like PGD, MD uses the gradient with respect to the transformed parameters. E.g. A MD-based update is $\bec{\gaVariable} \leftarrow \bec{\gaVariable} - \stepSize \grad_{\bec{\prediction}} f(\bec{\prediction})$, and transfer is done on-demand, $\bec{\prediction} \as \projection(\bec{\gaVariable})$.

Further difficulties are encountered when function approximation is involved, that is, when $\bec{\prediction} \as \projection(g(\bec{\gaVariable}))$, for an arbitrary function $g$. Now PGD's approach of making adjustments in the decision space where $\bec{\prediction}$ resides is untenable because the function parameters, $\bec{\gaVariable}$, might reside in a very different space. E.g. $\bec{\prediction}$ may be a complete strategy while $\bec{\gaVariable}$ may be a vector of neural network parameters with many fewer dimensions. But the gradient with respect to the $\bec{\prediction}$ is also in the decision space, so MD's update cannot be done exactly either.

A simple fix is to apply the chain rule to find the gradient with respect to $\bec{\gaVariable}$. This is the approach taken by \defword{policy gradient (PG)} methods~\cite{williams1992simple,Sutton00policy,Sutton18} (the ``all actions'' versions rather than sample action versions). A consequence of this choice, however, is that PG updates in the tabular setting (when $g$ is the identity function) generally do not reproduce MD updates.

E.g. \defword{hedge}, \defword{exponentially weighted experts}, or \defword{entropic mirror descent}~\cite{hedge,Beck2003entropicDescent}, is a celebrated algorithm for approximately solving games. It is a simple no-regret algorithm that achieves the optimal regret dependence on the number of actions, $\log \numActions$, and it can be used in the CFR and CFR-BR framework instead of the more conventional regret-matching to solve sequential imperfect information games. It is also an instance of MD, which we show here.

Hedge accumulates values associated with each action (e.g. counterfactual values or regrets) and transfers them into the probability simplex with the softmax function to generate a policy. Formally, given a sequence of values $[\bec{\expertValue}^t]_{t=1}^T, \bec{\expertValue}^t \as [\expertValue_{\action}^t]_{\action \in \Actions}$ and temperature, $\temperature > 0$, hedge plays
\begin{align*}
  \policy_{\text{hedge}}^T
    &= \frac{
        \e{\frac{1}{\temperature}\sum_{t=1}^T \bec{\expertValue}^t}
      }{
        \sum_{\action \in \Actions} \e{\frac{1}{\temperature}\sum_{t=1}^T \expertValue^t_{\action}}
      }.
\end{align*}

We now show how to recover this policy creation rule with MD.

Given a vector of bounded action values, $\bec{\expertValue} \in \expertValueSet^{\numActions} \subseteq \reals^{\numActions}, \max \expertValueSet - \min \expertValueSet \leq \utilityRange$, the expected value of policy $\bec{\policy}_{\text{MD}}$ interpreted as a probability distribution is just the weighted sum of values, $\eExpertValue \as \sum_{\action \in \Actions} \policy_{\text{MD}}(\action) \expertValue_{\action}$.

The gradient of the expected value of $\bec{\policy}_{\text{MD}}$'s value is then just the vector of action values, $\bec{\expertValue}$. MD accumulates the gradients on each round,
\begin{align*}
    \bec{\gaVariable}^{t} 
        &\as \bec{\gaVariable}^{t - 1} + \stepSize \grad_{\bec{\policy}^{t - 1}} \eExpertValue^{t}\\
        &= \bec{\gaVariable}^{t - 1} + \stepSize \bec{\expertValue}^{t}\\
        &= \bec{\gaVariable}^{0} + \stepSize \sum_{i = 1}^{t} \bec{\expertValue}^{i},
\end{align*}
where $\stepSize > 0$ is a step-size parameter. If $\bec{\gaVariable}^{0}$ is zero, then the current parameters, $\bec{\gaVariable}^{t}$, are simply the step-size weighted sum of the action values.

If $\bec{\policy}_{\text{MD}}^t$ on each round, $t \geq 0$ is chosen to be $\bec{\policy}_{\text{MD}}^t = \softmaxTransfer(\bec{\gaVariable}^t)$, then we can rewrite this policy in terms of the action values alone:
\begin{align*}
    \bec{\policy}_{\text{MD}}^t
        &= \frac{
        \e{\stepSize \sum_{t=1}^T \bec{\expertValue}^t}
      }{
        \sum_{\action \in \Actions} \e{\stepSize \sum_{t=1}^T \expertValue^t_{\action}}
      },
\end{align*}
which one can recognize as hedge with $\temperature \as 1 / \stepSize$.
This shows how hedge fits into the ED framework. When counterfactual values are used for action values, ED with MD gradient-updates and a softmax transfer at every information state is identical to CFR-BR with hedge at every information state.

In comparison, PG, using the same transfer and tabular parameterization, generates policies according to
\begin{align*}
    \bec{\gaVariable}^t 
        &\as \bec{\gaVariable}^{t - 1} + \stepSize \ip{\grad_{\bec{\gaVariable}^{t - 1}} \bec{\policy}_{\text{PG}}^{t - 1}}{\bec{\expertValue}^t}\\
    \policy_{\text{PG}}^t
        &\as \softmaxTransfer(\bec{\gaVariable}^t).
\end{align*}
$\dfrac{\partial \policy(\action')}{\partial \gaVariable_{\action}} = \policy(\action') \subblock{ \indicator{\action' = \action} - \policy(\action) }$~\cite[Section 2.8]{Sutton18}, so the update direction, $\ip{\grad_{\bec{\gaVariable}} \bec{\policy}}{\bec{\expertValue}}$, is actually the regret scaled by $\bec{\policy}$:
\begin{align*}
    \frac{\partial \bec{\policy}}{\partial \gaVariable_{\action}} \cdot \bec{\expertValue}
        &= \sum_{\action' \in \Actions}
            \frac{\partial \policy(\action')}{\partial \gaVariable_{\action}} \expertValue_{\action'}\\
        &= \sum_{\action' \in \Actions}
            \policy(\action') \subblock{ \indicator{\action' = \action} - \policy(\action) } \expertValue_{\action'}\\
        &= \policy(\action) \subblock{ 1 - \policy(\action) } \expertValue_{\action}
            - \sum_{\action' \in \Actions, \action' \neq \action}
                \policy(\action') \policy(\action) \expertValue_{\action'}\\
        &= \policy(\action) \subblock{
            \subblock{ 1 - \policy(\action) } \expertValue_{\action}
            - \sum_{\action' \in \Actions, \action' \neq \action}
                \policy(\action') \expertValue_{\action'}}\\
        &= \policy(\action) \subblock{
            \expertValue_{\action} - \policy(\action) \expertValue_{\action}
            - \sum_{\action' \in \Actions, \action' \neq \action}
                \policy(\action') \expertValue_{\action'}}\\
        &= \policy(\action) \subblock{
            \expertValue_{\action} 
            - \sum_{\action' \in \Actions} \policy(\action') \expertValue_{\action'}}.
\end{align*}
Knowing this, we can write the $\bec{\policy}_{\text{PG}}^t$ in concrete terms:
\begin{align*}
    \gaVariable^t_{\action} 
        &\as \gaVariable^{t - 1}_{\action} + \stepSize \policy_{\text{PG}}^{t - 1}(\action) \subblock{
            \expertValue_{\action}
            - \sum_{\action' \in \Actions} \policy_{\text{PG}}^{t - 1}(\action') \expertValue_{\action'}
        }\\
    \bec{\policy}_{\text{PG}}^t
        &\as \softmaxTransfer(\bec{\gaVariable}^t).
\end{align*}
The fact that the PG parameters accumulate regret instead of action value is inconsequential because the difference between action values and regrets is a shift that is shared between each action, and the softmax function is shift-invariant. But there is a substantive difference in that updates are scaled by the current policy.

\section{Global Minimum Conditions}
\label{GMC}
In this section we will suppose that the policy $\pi_i(s,a) = \theta_{s,a}$ under the simplex constraints $\forall s,a, \; \theta_{s,a} - \psi_{s,a}^2 = 0$ and $\forall s, \; \sum_a \theta_{s,a} = 1$ (where $psi$ is a slack variable to enforce the inequality constrain $\forall s,a, \; \theta_{s,a} \geq 0$)

\begin{equation*}
\begin{aligned}
& \underset{x}{\text{maximize}}&\min_{\pi_{-i}} v_{i,(\pi_{\theta_i}, \pi_{-i})} \\
& \text{subject to}& \sum_a \theta_{s,a} = 1, \; \forall s\\
& & \theta_{s,a} - \psi_{s,a}^2 = 0, \; \forall s,a\\
\end{aligned}
\end{equation*}

The Lagrangian is:
\begin{align*}
L(\theta, \psi, \lambda) =& \min_{\pi_{-i}} v_{i,(\pi_{\theta_i}, \pi_{-i})} - \sum \limits_{s\in S^i}\lambda_s(\sum_a \theta_{s,a}- 1)\\
&\;\; - \sum \limits_{s\in S^i, a}\lambda_{s, a}(\theta_{s,a} - \psi_{s,a}^2)
\end{align*}
Knowing that for all $br_{-i}$ :
$$\frac{\partial v_{i,(\pi_{\theta_i}, br_{-i})}}{\theta_{s,a}} = \eta_{(\pi_{\theta_i}, br_{-i})}(s) q_{i, (\pi_{\theta_i}, br_{-i})}(s, a)$$
The gradient of the Lagrangian with respect to $\theta, \psi, \lambda$ is:

\begin{align}
&\left(\eta_{(\pi_{\theta_i}, br_{-i})}(s) q_{i, (\pi_{\theta_i}, br_{-i})}(s, a) -(\lambda_{s} + \lambda_{s,a})\right)_{s,a}\nonumber\\
&\qquad\qquad\qquad\qquad\qquad\qquad\qquad\in \subdiff_{\theta_i} L(\theta, \psi, \lambda)\\
&\frac{\partial L(\theta, \psi, \lambda)}{\partial \lambda_{s}} = - (\sum_a \theta_{s,a} - 1)\\
&\frac{\partial L(\theta, \psi, \lambda)}{\partial \lambda_{s,a}} = - (\theta_{s,a} - \psi_{s,a}^2)\\
&\frac{\partial L(\theta, \psi, \lambda)}{\partial \psi_{s,a}} = 2 \psi_{s,a} \lambda_{s,a} \\
\end{align}

Suppose that there exists a best response $br_{-i}$ such that $\frac{\partial L(\theta, \psi, \lambda)}{\partial \theta_{s,a}} = \frac{\partial L(\theta, \psi, \lambda)}{\partial \lambda_{s}} = \frac{\partial L(\theta, \psi, \lambda)}{\partial \lambda_{s,a}} = \frac{\partial L(\theta, \psi, \lambda)}{\partial \psi_{s,a}} = 0$ (\textit{i.e.} if 0 is in the set of generalized gradients). Two cases can appear:

-If $\theta_{s,a}>0$ then $\lambda_{s,a} = 0$ and then:
$$\forall s, a | \theta_{s,a} > 0 \; \eta_{(\pi_{\theta_i}, br_{-i})}(s) q_{i, (\pi_{\theta_i}, br_{-i})}(s, a) = \lambda_s$$

-If $\theta_{s,a}=0$ then $\psi_{s,a} = 0$ and then:
    $$\eta_{(\pi_{\theta_i}, br_{-i})}(s) q_{i, (\pi_{\theta_i}, br_{-i})}(s, a) = \lambda_s + \lambda_{s,a}$$
Two cases (one stable and one unstable):
\begin{itemize}
    \item $\lambda_{s,a} \leq 0$ then we have a stable fixed point,
    \item $\lambda_{s,a} > 0$ is not stable as then we could increase the value by switching to that action.
\end{itemize}

We conclude by noticing that $\pi_{\theta_i}$ is greedy with respect to the value of the joint policy $(\pi_{\theta_i}, br_{-i})$, thus $\pi_{\theta_i}$ is a best response to $br_{-i}$. Since both policies are best responses to each other, $(\pi_{\theta_i}, br_{-i})$ is a Nash equilibrium. $\pi_{\theta_i}$ is also therefore unexploitable.

\section{Corrections to the Original Paper}
\label{app:errata}

We have modified this paper from its original version as a result of changes made to the theorem statements in~\cite{srinivasan2018actor}.
In particular, the $O(T)$ term was added to statements of Thereoms \ref{thm:edq} and \ref{thm:spgpi-br}.
Claims of TabularED($q$, $\ell_2$)'s guaranteed convergence to equilibria were removed. 
The claim of convergence of TabularED($q^c$, $\ell_2$) was unaffected since it reduces to CFR-BR(GIGA).
We have also modified the text to clarify the convergence properties.

\end{document}